\documentclass{article}
\usepackage{graphicx} 
\usepackage{subcaption} 
\usepackage{amssymb} 
\usepackage{amsmath}
\usepackage{mathrsfs}  
\usepackage{biblatex}
\usepackage{amsthm}
\usepackage{algorithm} 
\usepackage{algpseudocode}
\newtheorem{theorem}{Theorem}
\newtheorem{lemma}[theorem]{Lemma}
\newtheorem{proposition}[theorem]{Proposition}

\newtheorem*{remark}{Remark}
\usepackage{xcolor} 

\title{Selective Conformal Risk Control}

\author{%
\textbf{Yunpeng Xu}$^{1}$ \and
\textbf{Wenge Guo}$^{2}$ \and
\textbf{Zhi Wei}$^{1}$\\[0.35em]
\small $^{1}$Department of Computer Science, New Jersey Institute of Technology\\
\small $^{2}$Department of Mathematical Sciences, New Jersey Institute of Technology
}
       
\begin{document}

\maketitle

\begin{abstract}
Reliable uncertainty quantification is essential for deploying machine learning systems in high-stakes domains.  
Conformal prediction provides distribution-free coverage guarantees but often yields overly large prediction sets, limiting its practical utility.  
To address this limitation, we propose \textit{Selective Conformal Risk Control} (SCRC), a unified framework that integrates conformal prediction with selective classification.  
SCRC operates in a selective setting: it abstains on low-confidence inputs and applies conformal risk control only to the accepted subset, where compact and informative prediction sets are most valuable.
The framework formulates uncertainty control as a two-stage procedure.  
The first stage selects confident samples, while the second stage constructs calibrated prediction sets via conformal risk control.  
We develop two variants: SCRC-T, a transductive method that preserves exchangeability and achieves exact finite-sample guarantees at the cost of per-test recomputation, and SCRC-I, an inductive alternative that reuses calibration thresholds to provide PAC-style guarantees with improved computational efficiency.
Experiments on two benchmark datasets demonstrate that both methods achieve the desired coverage and risk levels with nearly identical performance.  
SCRC-I is slightly more conservative but substantially more practical for deployment.  
Overall, SCRC improves the efficiency of uncertainty quantification on accepted samples while explicitly deferring uncertain cases for abstention or downstream handling.
\end{abstract}

\section{Introduction}
Despite growing research on uncertainty quantification, many real-world machine learning systems still provide predictions without transparent confidence assessments. In high-stakes domains such as medical diagnosis, autonomous driving, or financial decision-making, this absence of uncertainty information makes it difficult to assess how much one can trust a given prediction, potentially leading to critical failures. Therefore, quantifying and controlling predictive uncertainty has become a key requirement for building reliable machine learning systems.

To address this need, a wide range of methods have been developed to estimate and communicate predictive uncertainty, spanning both Bayesian approaches such as \cite{blundell2015weight, GalGhahramani2016Dropout}, and calibration-based techniques such as  \cite{guo2017calibration, kuleshov2018accurate, vaicenavicius2019evaluating}. 

As a popular uncertain quantization method, conformal prediction \cite{Angelopoulos2021, angelopoulos2022conformal, fisch2021conformal} received a lot of attention in recent years. The method produces a prediction set by post-processing the model output such that a bigger prediction size corresponds to a less confident inference and vice versa. It is distribution-free, with finite-sample and model-agnostic coverage guarantees. 

On the other hand, however, conformal prediction often suffers from inefficiency in practice: to meet the desired coverage level, the resulting prediction sets can be excessively large, sometimes including most or all possible labels. This greatly limits its usefulness in real-world applications, for example medical diagnosis, where compact and actionable prediction sets are critical. Reducing the prediction set size without compromising the risk guarantee is therefore a central challenge in improving the practicality of conformal prediction.

A related and complementary approach is selective classification \cite{chow1957, hellman1970, Geifman2017, Geifman2019}, which has been well studied in machine learning field for several decades. It introduces a reject option: the model abstains from making a prediction when it is uncertain. By adjusting the rejection threshold, the model can trade off coverage (fraction of predictions made) against accuracy (risk). 

Selective classification provides a natural mechanism to improve the efficiency of conformal prediction. Instead of producing large prediction sets for all samples and attempting to provide improved uncertainty quantification uniformly, the model can abstain on uncertain inputs and focus on the accepted subset. This allows for smaller average prediction set sizes on the selected (high-confidence) samples while maintaining overall risk control. Combining the coverage guarantee of conformal prediction with the flexibility of selective classification can thus yield more efficient and practically useful uncertainty quantification.

In this work, we propose a unified framework, Selective Conformal Risk Control (SCRC), that integrates conformal prediction with selective classification. The proposed method introduces a two-stage risk control procedure: the first stage determines which samples to accept for prediction (selection control), and the second stage constructs conformal prediction sets for the accepted samples (risk control). This combination allows the model to produce informative prediction sets only when confident and to abstain otherwise, leading to a more compact and interpretable uncertainty representation.  Accordingly, the guarantees in our framework are selective: they apply to the accepted subpopulation, while rejected cases are intentionally deferred rather than assigned a possibly uninformative prediction set.

We further develop two algorithmic variants to serve complementary roles: SCRC-T, a transductive method that computes thresholds symmetrically over calibration and test data, with strict exchangeability guarantees; SCRC-I, a more computationally efficient variant that reuses calibration thresholds across test samples, with PAC-style high-probability guarantees. Both methods ensure risk and coverage control, with theoretical guarantees derived under the conformal risk control framework. Empirical evaluations demonstrate that our proposed methods achieve the desired risk and coverage levels while significantly reducing prediction set sizes compared to standard conformal prediction.

To summarize, this paper makes the following contributions:
\begin{itemize}
    \item We propose Selective Conformal Risk Control, by combining conformal risk control and selective classification into a unified two-stage uncertainty quantification framework.
    \item We develop two practical algorithms: SCRC-T and SCRC-I, which ensure valid risk and coverage control under different computational trade-offs.
    \item We establish theoretical results for the proposed methods.
    \item We provide extensive experiments to validate the proposed methods.
\end{itemize}

\section{Problem Formulation}

Selective classification provides a mechanism for models to abstain from making uncertain predictions, allowing a trade-off between coverage and accuracy. 
Given an input space $\mathcal{X}$ and label space $\mathcal{Y} = \{1,\dots,K\}$, let $f:\mathcal{X} \rightarrow [0,1]^K$ denote a base classifier producing class scores for $K$ classes, and $g:\mathcal{X}\rightarrow [0,1]$ denote a selection function that measures confidence. 
For a selection threshold $\lambda\in[0,1]$, the model outputs
\begin{equation}
\label{eq1}
\hat{y}(x) = 
\begin{cases}
f(x), & \text{if } g(x) \ge 1 - \lambda,\\
\circledR, & \text{otherwise,}
\end{cases}
\end{equation}
where $\circledR$ indicates rejection.  
The overall prediction coverage and conditional classification risk are defined respectively as
\begin{align}
\phi(f,g) &= \mathbb{P}(g(X) \ge 1 - \lambda), \label{eq2}\\
R(f,g) &= \mathbb{E}[\,l(f(X),Y)\mid g(X)\ge 1 - \lambda\,], \label{eq3}
\end{align}
where $l$ is a bounded loss function measuring the classification error.  
A conventional selective classifier often seeks an optimal $g(x)$ that minimizes the conditional risk under a coverage constraint:
\begin{equation}
\label{eq4}
\min_g \; R(f,g) \quad \text{s.t. } \phi(f,g)\ge \xi.
\end{equation}

In this work, we are not concerned with learning $f$ or $g$ directly. 
Instead, we assume both are fixed and focus on calibrating their outputs to guarantee the desired risk and coverage levels under the conformal prediction framework. 
To that end, we introduce the \textit{Selective Conformal Risk Control} problem.

\paragraph{Selective Conformal Risk Control.}
Let $\mathbf{Z}=\{(x_i,y_i)\}_{i=1}^n$ be an exchangeable calibration dataset, where each sample is associated with model outputs $(f_i,g_i)$.  
For a new test instance $X_{n+1}$ and its corresponding score pair $(f_{n+1},g_{n+1})$, we determine its selective prediction $\mathcal{C}(X_{n+1}) \in \{\circledR\}\cup 2^{\mathcal{Y}}$, by post-processing the model score pair $(f_{n+1},g_{n+1})$ using calibration data, where  $\mathcal{C}$ is a function of the models as well as the calibration set.

Specifically, we introduce two calibration thresholds $\lambda=(\lambda_1,\lambda_2)$, and decide the output of $X_{n+1}$ as:
\begin{equation}
\mathcal{C}(X_{n+1}) =
\begin{cases}
\circledR, & g(X_{n+1})< 1-\lambda_1,\\
C_{\lambda_2}(X_{n+1}), & \text{otherwise.}
\end{cases}
\end{equation}
where $C_{\lambda_2}(X_{n+1})$ denotes a prediction set constructed according to threshold $\lambda_2$.  
In this work we use
\[
C_{\lambda_2}(X_{n+1}) = \{k\in\{1,\dots,K\}: f(X_{n+1})_k \ge 1-\lambda_2\},
\]
although other set-construction rules could also be employed.

The first-stage threshold $\lambda_1$ controls which examples are accepted for prediction, while the second-stage threshold $\lambda_2$ determines the size of the prediction set for the accepted cases.  
We evaluate selective performance using
\begin{align}
R(f,g) &= \mathbb{E}\big[l(\mathcal{C}_{\lambda_2}(X_{n+1}),Y_{n+1})\mid g(X_{n+1})\ge 1-\lambda_1\big], \\
\phi(f,g) &= \mathbb{P}(g(X_{n+1})\ge 1-\lambda_1),
\end{align}
where $l(\mathcal{C}_{\lambda_2}(X_{n+1}),Y_{n+1})\in[0,1]$ is a bounded, monotonically decreasing loss function that diminishes as the prediction set $\mathcal{C}_{\lambda_2}(X_{n+1})$ expands.

Since the expected set size conditional on acceptance is an indicator of prediction efficiency, our goal is therefore to find calibration parameters $(\lambda_1,\lambda_2)$ that satisfy the selective coverage $\xi$ and risk $\alpha$ requirements, while minimizing the expected prediction set size:
\begin{equation}
\label{eq5}
\begin{aligned}
\min_{(\lambda_1,\lambda_2)} \; & \mathbb{E}\big[|\mathcal{C}_{\lambda_2}(X_{n+1})| \mid g(X_{n+1})\ge 1-\lambda_1\big] \\
\text{s.t. } & R(f,g)\le \alpha, \quad \phi(f,g)\ge \xi.
\end{aligned}
\end{equation}

This defines the Selective Conformal Classification Problem.
The minimization of the conditional prediction set size prevents trivial solutions (e.g., always predict or include all labels), ensuring an efficient and informative prediction.

\section{Related Work}

\paragraph{Conformal Prediction.}
Conformal prediction (CP), originally developed by Vovk and colleagues \cite{Vovk1999,vovk2005}, is a foundational framework for distribution-free uncertainty quantification in machine learning.  
It provides finite-sample coverage guarantees under the assumption of data exchangeability, enabling rigorous uncertainty calibration without relying on parametric assumptions.  
A wide range of CP variants have been developed, including inductive and split conformal prediction \cite{papadopoulos2002inductive,lei20183distribution}, conformalized quantile regression \cite{romano2019conformalized}, and covariate-shift adaptation \cite{tibshirani2019covshift}, as well as limited false positives control \cite{fisch2021conformal}.  Comprehensive overviews of this field can be found in \cite{Angelopoulos2021,shafer2008tutorial}, which review both theoretical foundations and practical applications across regression, classification, and structured prediction tasks.

Recent research has reframed CP through the lens of \emph{risk control}\cite{angelopoulos2022conformal}: instead of targeting a coverage guarantee, conformal risk control (CRC) directly constrains the expected loss of the prediction at a target level $\alpha$. This generalization extends the scope of the classical CP and enables it for many new applications. Our work is conceptually aligned with this framework. 

\paragraph{Selective Classification.} Selective classification (also known as classification with a reject option) has received extensive studies in the past few decades. 
The foundational studies by \cite{chow1957, hellman1970} established optimal decision rules under the reject option.  
Since then, numerous selective classification methods have been introduced, including \cite{bartlett2008reject, Geifman2017, Geifman2019, franc2023optimalreject, pugnana2023aucselect}. 
These methods provide a principled trade-off between accuracy and coverage, by allowing models to abstain from making a prediction when it is uncertain.
A comprehensive survey of machine learning approaches with rejection is provided by \cite{hendrickx2024mlrejectsurvey}.

\paragraph{Conformal Prediction and Selective Classification.} A small but growing body of research integrates conformal prediction or calibration with selective classification principles.
\cite{fisch2022} developed a calibrated selective classification framework that trains a selection model to ensure that accepted predictions remain probability-calibrated. \cite{bao2024scifcsr} introduced selective conditional conformal prediction with false coverage rate (FCR) control.   \cite{gazin2024informative} further developed two informative selective CP procedures that guarantee FCR while constraining informativeness. 
In contrast, our method addresses a distinct but complementary problem: we formulate a two-stage selective conformal risk-control framework that simultaneously enforces coverage and conditional risk guarantees on accepted samples, while minimizing the expected prediction set size to improve efficiency.

\section{Method}
The problem naturally involves two stages, selection and classification, each targeting a distinct risk control objective. This formulation is conceptually related to the two-stage risk control framework introduced in \cite{yunpengxu2025twostages}. However, the selective setting presents a critical challenge: the act of selection disrupts exchangeability between calibration and test samples, a key condition for applying the conformal risk control (CRC) framework. As a result, the existing two-stage method cannot be directly applied without modification. We therefore develop a new approach tailored to selective classification.

\subsection{Conditional Exchangeability}
To invoke standard results from the CRC framework, we first need to address the question of exchangeability after selection.

\begin{lemma}
\label{lemma1}
    Suppose $\left(X_1, Y_1\right), \ldots, \left(X_{n+1}, Y_{n+1}\right)$ are exchangeable, and let $\mathcal{I}$ be a symmetric selection rule, meaning that for any dataset $\mathcal{D}$, any permutation $\sigma$ of $[n+1]$, and any index $i \in [n+1]$,
    $$
    \sigma(i) \in \mathcal{I}(\mathcal{D}) \Longleftrightarrow i \in \mathcal{I}\left(\mathcal{D}^\sigma\right),
    $$
    where $\mathcal{D}^\sigma = \{(X_{\sigma(i)}, Y_{\sigma(i)})\}_{i=1}^{n+1}$ denotes the permutated dataset. 
    Let $\mathcal{E}_I$ denote the event that $\mathcal{I}\left(\mathcal{D}\right)=I$, for some fixed nonempty subset $I \subseteq[n+$ $1]$, and assume $\mathbb{P}\left(\mathcal{E}_I\right)>0$. Then, conditional on $\mathcal{E}_I$, the subcollection $\left\{\left(X_i, Y_i\right)\right\}_{i \in I}$ is exchangeable. 
\end{lemma}

\begin{proof}
Let $A \subset (\mathcal{X} \times \mathcal{Y})^{|I|}$ be a measurable set of values for the data pairs. Let $\pi$ be any permutation of the index set $I$, and extend it to a permutation $\sigma$ on the full set of indices $[n+1]$ as follows, 
\[
\sigma(i)=
\begin{cases}
\pi(i) & \text{for $i \in I$} \\
i & \text{for $i \notin I$}
\end{cases}
\]
Since $I$ is a set, we have $\sigma(I)=I$. We aim to show that
\[
P((X_i, Y_i)_{i\in I} \in A | \mathcal{E}_I) = P((X_{\pi(i)}, Y_{\pi(i)})_{i\in I} \in A | \mathcal{E}_I).
\]
Start from the numerator of the conditional probability, and write it as an expectation of an indicator:
\[
P((X_i, Y_i)_{i\in I} \in A, \mathcal{E}_I) = E[\mathbf{1}_{\{(X_i, Y_i)_{i\in I} \in A\}}\mathbf{1}_{\mathcal{E}_I}].
\]
By exchangeability of the full dataset, we have 
\[
E[\mathbf{1}_{\{(X_i, Y_i)_{i\in I} \in A\}}\mathbf{1}_{\mathcal{E}_I}] = E[\mathbf{1}_{\{(X_{\sigma(i)}, Y_{\sigma(i)})_{i\in I} \in A\}}\mathbf{1}_{\mathcal{E}_{I_\sigma}}],
\]
where $\mathcal{E}_{I_\sigma}$ is the event that $\mathcal{I}(D^\sigma) =I$. By symmetry of the selection rule $\mathcal{I}$, we have $\mathcal{I}(D^\sigma)= \sigma(\mathcal{I}(D))$. Since $\sigma(I)=I$, it follows that $\mathbf{1}(\mathcal{E}_{I_\sigma}) = \mathbf{1}(\mathcal{E}_I)$, therefore,
\[
P((X_i, Y_i)_{i\in I} \in A,{\mathcal{E}_I}) = P((X_{\sigma(i)}, Y_{\sigma(i)})_{i\in I} \in A,{\mathcal{E}_I}) = P((X_{\pi(i)}, Y_{\pi(i)})_{i\in I} \in A,{\mathcal{E}_I}).
\]
\end{proof}

\subsection{First Stage Control}
Lemma \ref{lemma1} says that to preserve exchangeability, the selection rule employed in the first stage must be symmetric. This requirement motivates a modification to the standard first-stage risk control procedure.

In the first stage, the loss function depends solely on the feature $X$, rather than on the full data pair $(X, Y)$. The threshold $\lambda_1$ is determined by computing the empirical quantile of the feature-dependent loss values over both calibration and test samples, yielding the estimator $\hat{\lambda}_1$. Because this construction treats all features symmetrically, $\hat{\lambda}_1$ is a symmetric function of $(X_1, \ldots, X_{n+1})$. This design contrasts with the standard CRC method, where the threshold is estimated using only the calibration data and is independent of the test data, thus breaking symmetry and invalidating exchangeability in the selective setting. 

Formally, define the first-stage loss function as
\begin{equation}
\label{eq_6}
L^{(1)}(X;\lambda_1)=\mathbf{1}\{g(X)<1-\lambda_1\}.
\end{equation}
Given calibration features $(X_1,\dots,X_n)$ and test feature $X_{n+1}$, the empirical first stage risk is
\begin{equation}
\label{eq_7}
\widehat{R}_{n+1}^{(1)}(\lambda_1)=\frac{1}{n+1} \sum_{i=1}^{n+1} L^{(1)}(X_i, \lambda_1).
\end{equation}
The data-driven threshold is then defined by
    \begin{equation}
    \label{eq_8}
    \hat{\lambda}_1=\inf \left\{\lambda_1 \in \Lambda_1: \widehat{R}_{n+1}^{(1)}(\lambda_1) \leq 1 - \xi\right\} .
    \end{equation}

By this construction, $\hat{\lambda}_1$ is a symmetric function of the feature values $(X_1, \ldots, X_{n+1})$. Therefore, under the exchangeability assumption on $(X_1, \ldots, X_{n+1})$ and by Lemma \ref{lemma1}, the selection event defined through $\hat \lambda_1$ preserves exchangeability, ensuring that $\mathbb{P}(g(X_{n+1})\ge 1-\hat\lambda_1)\ge \xi$. This establishes valid first-stage risk control while maintaining the symmetry required for the second-stage analysis.

\subsection{Second Stage Control}
Once $\hat{\lambda}_1$ is determined, the selection rule $\mathcal{I}$ is defined as
\begin{equation}
\label{eq_9}
\mathcal{I}\left(X_1, \ldots, X_{n+1}\right)=\left\{i \in[n+1]: g\left(X_i\right) \geq 1-\hat{\lambda}_1\right\} .
\end{equation}
This rule is symmetric with respect to the inputs $(X_1, \ldots, X_{n+1})$, since it depends only on the symmetric threshold $\hat{\lambda}_1$ and applies the same elementwise comparison to each feature. 

By Lemma \ref{lemma1}, let $\mathcal{E}_I$ denote the event that $\mathcal{I}(X_1, \ldots, X_{n+1}) = I$ for some fixed nonempty subset $I \subseteq [n+1]$ with $n+1 \in I$, and assume $\mathbb{P}(\mathcal{E}_I) > 0$. Then, conditional on $\mathcal{E}_I$, the subcollection $\left\{(X_i, Y_i)\right\}_{i \in I}$ is exchangeable. 
Therefore, the standard CRC procedure can be applied to this selected subset to estimate the second-stage threshold $\hat{\lambda}_2$, ensuring conditional risk control at level $\alpha$.
    
Formally, for a fixed candidate value $\bar\lambda_1\in[0,1]$, define the selected calibration subset
\begin{equation}
\label{eq_9_1}
Z_{\bar\lambda_1}=\{(X_i,Y_i): g(X_i)\ge 1-\bar\lambda_1,\; i\in[n]\}.
\end{equation}
We denote its cardinality by $\quad m=|Z_{\bar\lambda_1}|$.
In practice, the search over $\bar\lambda_1$ is restricted to the interval $[\hat\lambda_1,1]$ to reduce computation.

Define the second-stage loss as
\begin{equation}
\label{eq_9_2}
L^{(2)}(X,Y;\lambda_2)=l(C_{\lambda_2}(X),Y),
\end{equation}
which is assumed to be bounded in $[0,1]$ and non-increasing in $\lambda_2$. 

The conformal risk control (CRC) rule chooses the largest feasible threshold
\begin{equation}
\label{eq_11}
    \hat \lambda_2 = inf \{\lambda_2 \in [0, 1]: \sum_{(X_i, Y_i) \in Z_{\bar \lambda_1}} l(\mathcal{C}_{\lambda_2}(X_i), Y_i) \le  \lceil (m+1) \alpha \rceil  -1 \}.
\end{equation}
This construction guarantees that
\begin{equation}
\mathbb{E}[l(C_{\hat\lambda_2}(X),Y) \mid g(X)\ge 1-\bar\lambda_1]\le \alpha.
\end{equation}

\begin{theorem}[Selective CRC Guarantee]
\label{thm:selective_crc}
{Assume $(X_i,Y_i)_{i=1}^{n+1}$ are exchangeable. Let $\hat\lambda_1$ be obtained by Equation~(\ref{eq_8}), let $\bar\lambda_1$ be a \emph{symmetric} function of the $n{+}1$ points such that $\bar\lambda_1\ge \hat\lambda_1$ almost surely, and let $\hat\lambda_2$ be obtained by Equation~(\ref{eq_11}). Then the resulting selective classifier satisfies}
\[
{\mathbb{E}[l(C_{\hat\lambda_2}(X_{n+1}),Y_{n+1})\mid g(X_{n+1})\ge 1-\bar\lambda_1] \le \alpha,}
\]
{and the selection coverage satisfies}
\[
{\mathbb{P}\!\left(g(X_{n+1})\ge 1-\bar\lambda_1\right)\ \ge\ \xi.}
\]
\end{theorem}

\begin{proof}
{(Coverage) Because $\hat\lambda_1$ is a symmetric functional of $(X_1,\dots,X_{n+1})$, hence by exchangeability}
\[
{\mathbb{E}[\mathbf{1}\{g(X_{n+1})<1-\hat\lambda_1\}]=\mathbb{E}[\hat R^{(1)}_{n+1}(\hat\lambda_1)]\le 1-\xi.}
\]
{Since $\bar \lambda_1 \ge \hat \lambda_1$ almost surely, monotonicity in $\lambda_1$ yields}
\[
{\mathbb{P}(g(X_{n+1})\ge 1-\bar\lambda_1)\ge \mathbb{P}(g(X_{n+1})\ge 1-\hat\lambda_1)\ge \xi.}
\]

{(Risk) Let $I=\{i:g(X_i)\ge 1-\bar\lambda_1\}$ and $m=|I\cap[n]|$. Since $\bar\lambda_1$ is symmetric, conditioning on $E_I=\{I(D_{n+1})=I,\ n{+}1\in I\}$, Lemma~\ref{lemma1} ensures that $\{(X_i,Y_i)\}_{i\in I}$ are exchangeable. Applying the CRC counting rule to the bounded, non-increasing loss $l(C_{\lambda_2},\cdot)$ and Equation~\eqref{eq_11} yields}
\[
{\mathbb{E}[l(C_{\hat\lambda_2}(X_{n+1}),Y_{n+1})\mid g(X_{n+1})\ge 1-\bar\lambda_1]\le \alpha.}
\]
\end{proof}

\begin{remark}[Feasibility check]
To ensure a valid solution of Equation (\ref{eq_11}), we requier $\lceil (m+1) \alpha \rceil  -1 > 0$, which holds whenever $m\ge m_{min} = \lceil 1/\alpha\rceil-1$. In practice, candidates with $m < m_{min}$ are skipped, or $\bar \lambda_1$ is increased, to ensure feasible risk control at level $\alpha$ in the second stage. 
\end{remark}

\subsection{Set-Size Refinement}
{Theorem~\ref{thm:selective_crc} establishes the selective validity of the base construction for a fixed admissible first-stage threshold. In practice, one may further reduce the prediction set size by searching over candidate values of $\lambda_1$ whenever a family of feasible candidates are available. 

For a fixed $\lambda_1 \in [\hat \lambda_1, 1]$, let $\lambda_2(\lambda_1)$ be the maximal value of $\lambda_2$ that satisfies Equation~\ref{eq_11}. By the monotonicity of the loss $L^{(2)}(X,Y;\lambda_2)$ in $\lambda_2$, the corresponding prediction set size $|C_{\lambda_2(\lambda_1)}(X)|$ is the smallest among all feasible $\lambda_2 \in [\lambda_2(\lambda_1), 1]$.

The first-stage threshold $\hat \lambda_1$ is the minimal threshold that guarantees the target selection coverage, but it need not minimize the conditional prediction set size. Moreover, the second-stage risk need not be monotone in $\lambda_1$ for fixed $\lambda_2$, so restricting attention to $\hat\lambda_1$ alone may be suboptimal from an efficiency standpoint. This motivates a practical search over a finite grid of candidate first-stage thresholds.

Accordingly, we consider the following heuristic: for each candidate $\lambda_1$ in a user-specified grid, construct the corresponding selected calibration subset, compute the feasible second-stage threshold, and then choose the candidate that yields the smallest empirical prediction set size. The resulting procedure is summarized in Algorithm~\ref{alg_1}, and its \emph{efficiency} aspect is formalized in Theorem~\ref{thm:erm-near-opt}. We emphasize that this procedure is intended as a practical refinement to improve set size; moreover, its output is not covered by Theorem~\ref{thm:selective_crc}.

\begin{algorithm}[tb]
\caption{Heuristic Search for Efficient SCRC Thresholds}
\label{alg_1}
\textbf{Input}: Calibration data $\{(X_i, Y_i)\}_{i = 1}^n$, test feature $X_{n+1}$, coverage level $\xi$, classification risk level $\alpha$, candidate grid $\Lambda \subseteq [0,1]$\\
\textbf{Output}: A practically chosen pair $(\bar{\lambda}_1, \bar{\lambda}_2)$
\begin{algorithmic}[1]
\State Compute the theorem-backed threshold $\hat{\lambda}_1$ using Equation~\ref{eq_8}.
\State Initialize $(\bar \lambda_1, \bar \lambda_2) \gets (\hat\lambda_1, +\infty)$.
\For{each $\lambda_1 \in \Lambda$ such that $\lambda_1 \ge \hat\lambda_1$}
        \State Construct the selected calibration subset $Z_{\lambda_1}$ using Equation~\ref{eq_9_1}.
        \State $m \gets |Z_{\lambda_1}|$.
        \If{$m > m_{\min}$}
            \State Compute the feasible second-stage threshold $\hat \lambda_2(\lambda_1)$ using Equation~\ref{eq_11}.
            \State Evaluate the empirical conditional prediction set size $\widehat S(\lambda_1)$.
            \If{$\widehat S(\lambda_1) < \widehat S(\bar\lambda_1)$}
                \State $(\bar \lambda_1, \bar \lambda_2) \gets (\lambda_1, \hat \lambda_2(\lambda_1))$.
            \EndIf
        \EndIf
\EndFor
\State \Return $(\bar{\lambda}_1, \bar{\lambda}_2)$
\end{algorithmic}
\end{algorithm}


\begin{theorem}[Finite-Grid Selection Efficiency]\label{thm:erm-near-opt}
Assume the calibration sample $(X_i,Y_i)_{i=1}^n$ is i.i.d.
Let $\Lambda\subset[0,1]$ be a finite, data-independent grid and define
$\hat\lambda_1^{\mathrm{ERM}}\in{\arg\min}_{\lambda_1\in\Lambda, \lambda_1\ge\hat\lambda_1}\widehat S(\lambda_1)$, where $\widehat S(\lambda_1)$ denotes the empirical mean prediction set size for parameter $\lambda_1$.
If the prediction set sizes are bounded by $B$ and $m_{\min}=\min_{\lambda_1\in\Lambda, \lambda_1\ge\hat\lambda_1} m(\lambda_1)$,
then for any $\delta\in(0,1)$, with probability at least $1-\delta$,
\[
S(\hat\lambda_1^{\mathrm{ERM}})
\;\le\;
\min_{\lambda_1\in\Lambda,\,\lambda_1\ge\hat\lambda_1} S(\lambda_1)
\;+\;
2B\sqrt{\tfrac{\log(2|\Lambda|/\delta)}{2\,m_{\min}}}.
\]
\end{theorem}

\begin{proof}
For a fixed $\lambda_1\in\Lambda$, under the i.i.d.\ assumption, the selected calibration subset
$Z_{\lambda_1}$ consists of $m(\lambda_1)$ i.i.d.\ samples drawn from the conditional distribution
$(X,Y)\mid g(X)\ge 1-\lambda_1$.
Hence $\widehat S(\lambda_1)$ is the empirical mean of $m(\lambda_1)$ independent
bounded random variables with expectation $S(\lambda_1)$.

By Hoeffding’s inequality,
\[
\Pr\!\left(\big|\widehat S(\lambda_1)-S(\lambda_1)\big|>\varepsilon\right)
\le 2\exp\!\Big(-\tfrac{2m(\lambda_1)\varepsilon^2}{B^2}\Big).
\]
Because the grid $\Lambda$ is deterministic, applying a union bound over $\Lambda$ and using
$m(\lambda_1)\ge m_{\min}$ yields
\[
\sup_{\lambda_1\in\Lambda}\big|\widehat S(\lambda_1)-S(\lambda_1)\big|
\le B\sqrt{\tfrac{\log(2|\Lambda|/\delta)}{2m_{\min}}}
\]
with probability at least $1-\delta$.
On this event, the empirical minimizer satisfies
\[
S(\hat\lambda_1^{\mathrm{ERM}})
\le \widehat S(\hat\lambda_1^{\mathrm{ERM}})+\varepsilon
\le \widehat S(\lambda_1^\star)+\varepsilon
\le S(\lambda_1^\star)+2\varepsilon,
\]
where
\[
\lambda_1^\star
=
\arg\min_{\lambda_1\in\Lambda,\ \lambda_1\ge \hat\lambda_1}
S(\lambda_1).
\]
Substituting $\varepsilon$ gives the claim.
\end{proof}


\begin{remark}
In practice, when the selection score $g(x)$ and the prediction score $f(x)$ are derived from the same underlying logits (as is common in deep classifiers), the two stages are strongly coupled: smaller $g(x)$ values also tend to have lower classification confidence under $f(x)$. Consequently, as $\lambda_1$ increases, the selected subset gradually includes more low-confidence and hence more difficult instances, often resulting in a monotone increase of the mean prediction set size.

However, this monotonic relationship is not guaranteed in general. When $g(x)$ and $f(x)$ are decoupled---for instance, if $g(x)$ measures some auxiliary notion of ``informativeness'' unrelated to the posterior confidence used in $f(x)$---the subsets selected by different $\lambda_1$ values may not correspond to strictly increasing difficulty levels. In such cases, the mapping between $\lambda_1$ and the mean set size can become non-monotone, justifying the need to perform a grid search over $\lambda_1$ rather than assuming monotonicity.
\end{remark}

\subsection{Calibration-only Variant}
In the transductive construction, the first-stage threshold depends symmetrically on both the calibration data and the test feature. This preserves exchangeability but requires recomputing the threshold for each new test point. For deployment, it is desirable to determine the first-stage threshold once from calibration data and reuse it for all future test points. We therefore introduce an inductive variant, SCRC-I, with a high-probability PAC-style guarantee.


Specifically, for any candidate first-stage threshold $\lambda_1$, define the selection indicator
\[
S_{\lambda_1}(x)=\mathbf{1}\{g(x)\ge 1-\lambda_1\}.
\]
Its population selection probability is
\[
q(\lambda_1)=\mathbb{P}\big(S_{\lambda_1}(X)=1\big),
\]
and its empirical counterpart on the calibration sample is
\[
\widehat q_n(\lambda_1)=\frac{1}{n}\sum_{i=1}^n S_{\lambda_1}(X_i).
\]

We choose the calibration-only first-stage threshold by
\begin{equation}
\hat\lambda_1'
=
\inf\left\{
\lambda_1\in\Lambda_1:\widehat q_n(\lambda_1)\ge \xi
\right\}.
\label{eq:inductive-lambda1}
\end{equation}

To lower-bound the true selection probability induced by the data-dependent threshold $\hat\lambda_1'$, we apply the Dvoretzky--Kiefer--Wolfowitz inequality with confidence level $\delta/2$. Define
\begin{equation}
\varepsilon_n^{(q)}
=
\sqrt{\frac{1}{2n}\log\frac{4}{\delta}},
\qquad
\xi_{\mathrm{LCB}}
=
\max\big\{\widehat q_n(\hat\lambda_1')-\varepsilon_n^{(q)},\,0\big\}.
\label{eq:xi-lcb}
\end{equation}

Next, for any candidate pair $(\lambda_1,\lambda_2)$, define the augmented population quantity
\begin{equation}
N(\lambda_1,\lambda_2)
=
\mathbb{E}\!\left[
S_{\lambda_1}(X)\,\ell(C_{\lambda_2}(X),Y)
\right],
\label{eq:aug-pop}
\end{equation}
and its empirical version
\begin{equation}
\widehat N_n(\lambda_1,\lambda_2)
=
\frac{1}{n}\sum_{i=1}^n
S_{\lambda_1}(X_i)\,\ell(C_{\lambda_2}(X_i),Y_i).
\label{eq:aug-emp}
\end{equation}

Since the loss is bounded in $[0,1]$, each summand in \eqref{eq:aug-emp} is also bounded in $[0,1]$. Over a finite candidate grid $\Lambda_1\times\Lambda_2$, a uniform Hoeffding inequality with confidence level $\delta/2$ yields
\begin{equation}
\varepsilon_n^{(N)}
=
\sqrt{
\frac{1}{2n}
\log\frac{4|\Lambda_1||\Lambda_2|}{\delta}
}.
\label{eq:eps-N}
\end{equation}

We then choose the second-stage threshold by the PAC-feasibility rule
\begin{equation}
\hat\lambda_2
=
\inf\left\{
\lambda_2\in\Lambda_2:
\widehat N_n(\hat\lambda_1',\lambda_2)+\varepsilon_n^{(N)}
\le
\alpha\,\xi_{\mathrm{LCB}}
\right\}.
\label{eq:inductive-lambda2}
\end{equation}

The selective conditional risk admits the decomposition
\begin{equation}
R(\lambda_1,\lambda_2)
=
\mathbb{E}\!\left[
\ell(C_{\lambda_2}(X),Y)\mid g(X)\ge 1-\lambda_1
\right]
=
\frac{N(\lambda_1,\lambda_2)}{q(\lambda_1)},
\label{eq:selective-risk-ratio}
\end{equation}
whenever $q(\lambda_1)>0$. The following proposition shows that the above construction yields a valid inductive guarantee.

\begin{proposition}[{Calibration-only variant}]
\label{prop:inductive-pac}
Assume $(X_i,Y_i)_{i=1}^n$ are i.i.d., let $\Lambda_1,\Lambda_2\subset[0,1]$ be finite grids, and construct $\hat\lambda_1'$ and $\hat\lambda_2$ according to \eqref{eq:inductive-lambda1} and \eqref{eq:inductive-lambda2}. Then for any $\delta\in(0,1)$, with probability at least $1-\delta$ over the calibration sample,
\[
\mathbb{E}\!\left[
\ell(C_{\hat\lambda_2}(X),Y)
\mid g(X)\ge 1-\hat\lambda_1'
\right]
\le \alpha,
\]
provided $\xi_{\mathrm{LCB}}>0$.
\end{proposition}

\begin{proof}
Apply the DKW inequality with confidence level $\delta/2$. Then with probability at least $1-\delta/2$,
\[
\sup_{\lambda_1\in\Lambda_1}
\big|\widehat q_n(\lambda_1)-q(\lambda_1)\big|
\le
\varepsilon_n^{(q)}.
\]
In particular,
\[
q(\hat\lambda_1')
\ge
\widehat q_n(\hat\lambda_1')-\varepsilon_n^{(q)}
=
\xi_{\mathrm{LCB}}.
\]

Next, for each fixed pair $(\lambda_1,\lambda_2)\in\Lambda_1\times\Lambda_2$, the variables
$S_{\lambda_1}(X_i)\,\ell(C_{\lambda_2}(X_i),Y_i)$ are i.i.d.\ and bounded in $[0,1]$. Hence Hoeffding's inequality and a union bound imply that, with probability at least $1-\delta/2$,
\[
\sup_{(\lambda_1,\lambda_2)\in\Lambda_1\times\Lambda_2}
\left|
\widehat N_n(\lambda_1,\lambda_2)-N(\lambda_1,\lambda_2)
\right|
\le
\varepsilon_n^{(N)}.
\]

On the intersection of these two events, the PAC-feasibility condition in \eqref{eq:inductive-lambda2} gives
\[
N(\hat\lambda_1',\hat\lambda_2)
\le
\widehat N_n(\hat\lambda_1',\hat\lambda_2)+\varepsilon_n^{(N)}
\le
\alpha\,\xi_{\mathrm{LCB}}
\le
\alpha\,q(\hat\lambda_1').
\]
Dividing both sides by $q(\hat\lambda_1')$ and using \eqref{eq:selective-risk-ratio} yields
\[
\mathbb{E}\!\left[
\ell(C_{\hat\lambda_2}(X),Y)
\mid g(X)\ge 1-\hat\lambda_1'
\right]
\le \alpha.
\]
A final union bound shows that this holds with probability at least $1-\delta$.
\end{proof}

\begin{remark}
SCRC-I is an inductive variant with a reusable calibration rule: both $\hat\lambda_1'$ and $\hat\lambda_2$ are computed once from the calibration sample and then reused for future test points. Its analysis does not require exchangeability after selection; instead, it controls the selection probability and the selected numerator risk separately via uniform concentration arguments. Relative to the transductive SCRC-T procedure, SCRC-I trades some statistical efficiency (due to concentration slack) for computational efficiency, avoiding per-test recomputation and making it more practical for deployment.
\end{remark}

\section{Experiments}
\subsection{Experiment Setup}
We adopt the following setup for all experiments presented in this section.  

\paragraph{Model and score preparation.}
Since model optimization is not the focus of this work, we use pretrained or standard models without additional hyperparameter tuning.  
The raw model logits are converted into calibrated class probabilities $f(x)$ via a temperature-scaled softmax transformation.  
For the selection function $g(x)$, we evaluate four commonly used confidence or uncertainty scores, all derived from the model logits:
\begin{itemize}
    \item \text{Maximum Softmax Probability (MSP)} \cite{hendrycks17baseline}: the highest predicted class probability,  $g_{\text{MSP}}(x) = \max_k p(y=k \mid x).$
    \item \text{Margin} \cite{hendrycks17baseline}: the difference between the top two predicted probabilities,  $g_{\text{margin}}(x) = p_{(1)} - p_{(2)}.$
    \item \text{Entropy} \cite{GalGhahramani2016Dropout}: the entropy of the predictive distribution, representing overall uncertainty,  $g_{\text{entropy}}(x) = -\sum_k {p(y=k|x)\log {p(y=k|x)}}$.
    \item \text{Energy} \cite{weitangliu2020}: an energy-based confidence score derived directly from unnormalized logits,  $g_{\text{energy}}(x) = -T \log \sum_k{ {\exp{\frac{\text{logit}_k(x)}{T}}}}$.
\end{itemize}

\paragraph{Compared Methods}
We compare four approaches for risk control:
\begin{itemize}
    \item \text{SCRC-T}: the transductive variant of our proposed method, which preserves full exchangeability.
    \item \text{SCRC-I}: the calibration-only (inductive) variant of our proposed method, which provides PAC-style probabilistic guarantees.
    \item \text{CRC-ALL}: a two-stage baseline that applies traditional conformal risk control (CRC) after selecting all data in the first stage.
    \item \text{RAND}: a random selection baseline that samples data at the target coverage rate before applying CRC in the second stage.
\end{itemize}

\paragraph{Evaluation protocol.}
To assess the validity of risk and coverage control, we conduct two complementary analyses.  
First, we fix the target risk level $\alpha$ and vary the desired coverage $\xi$ to evaluate selective coverage control.  
Then, we fix $\xi$ and vary $\alpha$ to evaluate risk control performance.  
For each setting, we report both the achieved coverage and empirical risk, as well as the average prediction set size.  
We further compare results for the subset of accepted (selected) samples and the rejected (unselected) samples to characterize efficiency.  

Each experiment is repeated 100 times with random sampling, and we report the averaged results across runs.  

All codes, including for both data generation and risk control, are publicly available at  
{\texttt{https://github.com/git4review/conformal\_selective\_classification}}.

\subsection{CIFAR-10 Dataset}

The CIFAR-10 dataset\footnote{\url{https://www.cs.toronto.edu/~kriz/cifar.html}} consists of 60,000 color images of size $32\times 32$, evenly distributed across 10 mutually exclusive object classes.  
Following standard practice, we use the 50,000 training images and 10,000 test images, with approximately 5,000 and 1,000 images per class, respectively.  
We further split the 50,000 training images into training, validation, and calibration sets using a 7:1:2 ratio.  
A ResNet-18 classifier \cite{He_2016_CVPR} is trained to produce logits for all experiments.

\paragraph{Coverage control.}
Figure~\ref{fig:cifar10_coverage_control_for_different_xi} shows empirical coverage and prediction set sizes as functions of the target coverage $\xi$, with the risk level fixed at $\alpha = 0.1$ and using the margin score as the selection function.  
All methods except CRC\_ALL successfully control coverage at the desired level.  
As expected, CRC\_ALL produces a coverage of $1.0$ regardless of $\xi$, since the first-stage selection always accepts all samples.

For both SCRC-T and SCRC-I, prediction set sizes for the selected subset increase with $\xi$, reflecting the strong coupling between the selection score $g(x)$ and the predictive scores $f(x)$.  
Moreover, selected samples by these two methods consistently exhibit substantially smaller prediction sets than unselected samples, confirming that the selection mechanism can effectively reject uncertain examples for classification.  
The RAND and CRC\_ALL baselines, on the other hand, show prediction set sizes that remain unchanged across different values of $\xi$, due to their lack of the selection process.
 
\begin{figure}[ht]
\centering
\begin{subfigure}{1.05\linewidth}
  \centering
  \includegraphics[width=\linewidth]{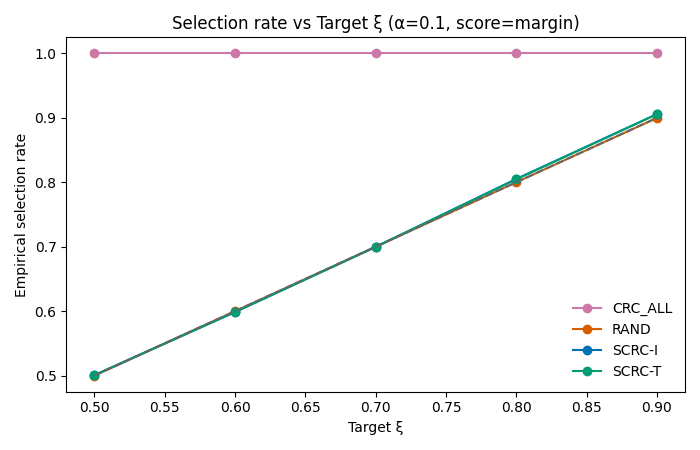}
\end{subfigure}
\begin{subfigure}{1.05\linewidth}
  \centering
  \includegraphics[width=\linewidth]{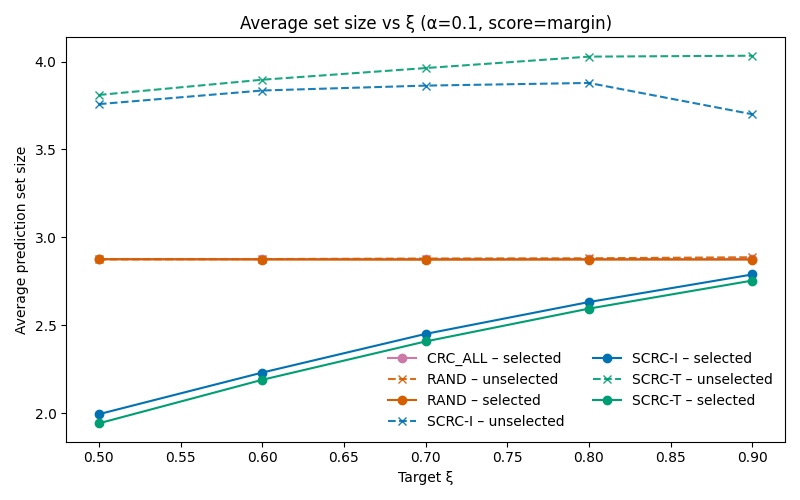}
\end{subfigure}
\caption{CIFAR-10: Coverage control at different values of $\xi$ with $\alpha=0.1$ (margin score).}
\label{fig:cifar10_coverage_control_for_different_xi}
\end{figure}

\paragraph{Risk control.}
Figure~\ref{fig:cifar10_risk_control_for_different_alpha} reports the empirical risk and prediction set sizes as functions of $\alpha$, with the target coverage fixed at $\xi = 0.7$.  
All methods achieve the desired selective risk control.  
Prediction set sizes decrease as $\alpha$ increases, consistent with the fact that looser risk constraints permit smaller prediction sets.  
As in the coverage experiments, selected samples exhibit much smaller prediction sets than rejected samples for both SCRC-T and SCRC-I,.

\begin{figure}[ht]
\centering
\begin{subfigure}{1.05\linewidth}
  \centering
  \includegraphics[width=\linewidth]{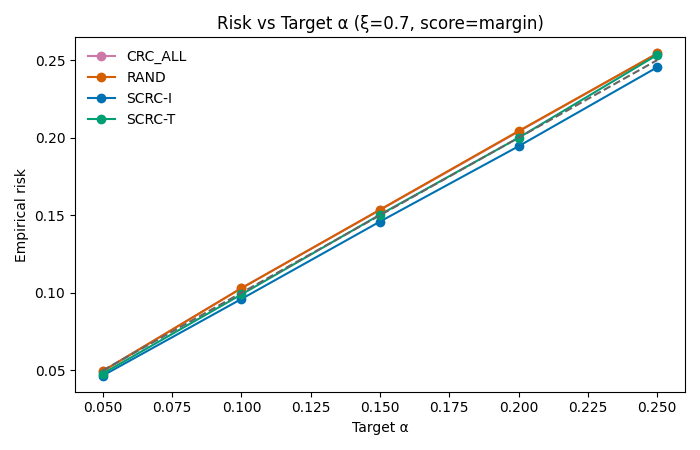}
\end{subfigure}
\begin{subfigure}{1.05\linewidth}
  \centering
  \includegraphics[width=\linewidth]{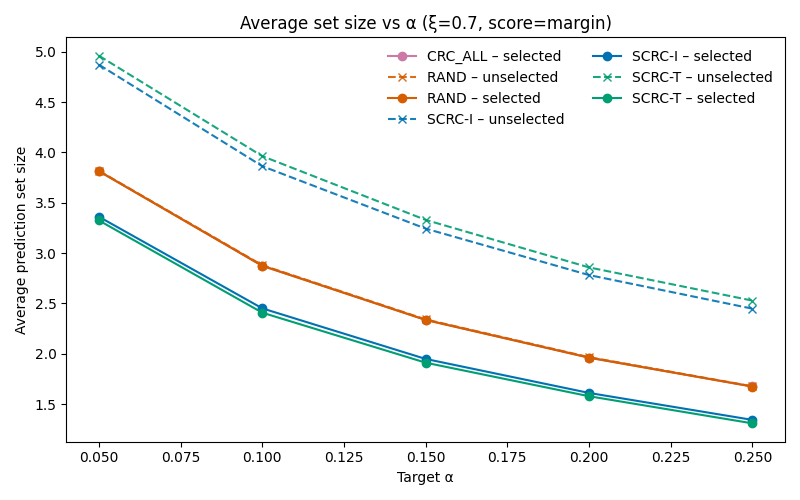}
\end{subfigure}
\caption{CIFAR-10: Risk control at different values of $\alpha$ with $\xi=0.7$ (margin score).}
\label{fig:cifar10_risk_control_for_different_alpha}
\end{figure}

\paragraph{Comparison of SCRC-T and SCRC-I.}
Across both sets of experiments, SCRC-T and SCRC-I deliver nearly identical empirical performance.  
However, SCRC-I is slightly more conservative in risk control, leading to marginally larger prediction sets.  
This behavior follows directly from its PAC-style correction using the lower confidence bound $\xi_{\mathrm{LCB}}$, which trades a small amount of efficiency for a reusable, inductive calibration procedure with high-probability guarantees.  
In practice, the difference between the two methods is minor and decreases as the calibration set grows.  
SCRC-T, on the other hand, yields a more efficient prediction sets when a per-test recomputation is feasible, though it incurs a much higher computational cost. This empirical comparison reflects their intended roles: SCRC-T serves as the exact exchangeability-preserving benchmark, while SCRC-I is the practically deployable variant.

\paragraph{Effect of score functions.}
To study the impact of different selection functions, we repeat the experiments using MSP, margin, entropy, and energy scores.  
As shown in Figure~\ref{fig:cifar10_score_functions}, all four choices achieve comparable coverage and risk control, but produce different prediction set sizes.  
Entropy and energy yield the smallest prediction sets, while margin produces the largest. This highlights an important distinction: the choice of $g(x)$ primarily affects \emph{efficiency}---how many examples are accepted and how small the resulting prediction sets are---rather than the validity of the selective guarantees themselves.

\begin{figure}[!htb]
  \centering
  \includegraphics[width=1.05\linewidth]{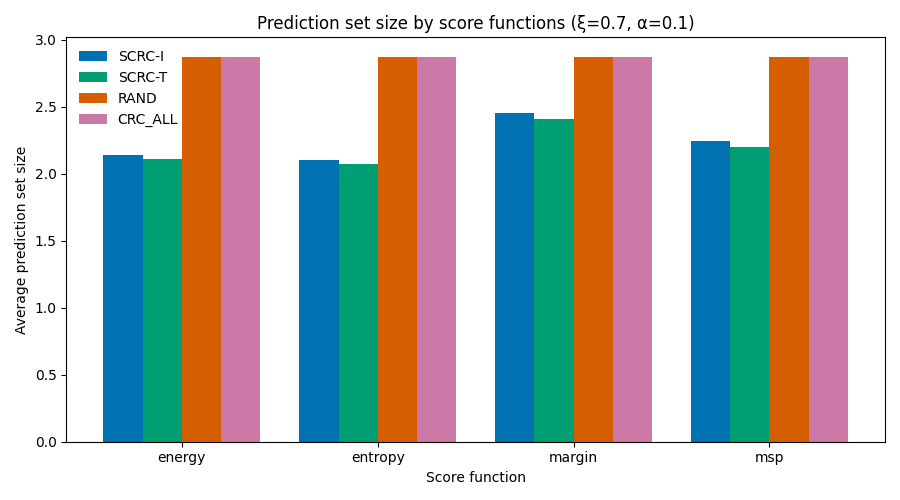}
  \caption{CIFAR-10: Comparison of different selection score functions.}
  \label{fig:cifar10_score_functions}
\end{figure}

\paragraph{Effect of $\delta$ for SCRC-I.}
Finally, we examine the impact of the confidence parameter $\delta$ on SCRC-I, which determines the DKW half-width $\epsilon_{n,\delta}$.  
Fixing $\xi = 0.7$ and $\alpha = 0.1$, we vary $\delta$ and report empirical results in Table~\ref{tab:cifar10_delta}.  
Tighter values of $\delta$ (i.e., stronger confidence guarantees) lead to slightly more conservative risk control and marginally larger prediction sets, both of which are expected behaviors.  
Overall, the effect is small but consistent.

\begin{table}[ht]
\centering
\caption{CIFAR-10: Impact of different values of $\delta$ for SCRC-I.}
\label{tab:cifar10_delta}
\setlength{\tabcolsep}{10pt}
\renewcommand{\arraystretch}{1.3}
\begin{tabular}{|l|c|c|c|}
\hline
\textbf{Metric} & \textbf{$\delta=0.01$} & \textbf{$\delta=0.05$} & \textbf{$\delta=0.10$} \\
\hline
Empirical risk        & 0.0954 & 0.0960 & 0.0964 \\
Prediction set size   & 2.460  & 2.451  & 2.446  \\
\hline
\end{tabular}
\end{table}

\subsection{Diabetic Retinopathy Detection Dataset}
The Diabetic Retinopathy Detection (DRD) dataset\footnote{\url{https://www.kaggle.com/competitions/diabetic-retinopathy-detection}} contains over 35,000 retinal fundus images captured under varying imaging conditions.  
Each image is manually graded by clinicians into one of five ordinal severity levels (0 to 4 where 4 is the most severe level), making the task clinically meaningful and well suited for uncertainty-aware evaluation.  
Compared to CIFAR-10, this dataset presents two notable differences: (i) it comes from a medical imaging domain where reliable uncertainty quantification is particularly critical, and (ii) the labels are \emph{ordinal}, requiring a different notion of prediction-set loss.  
We therefore adopt the weighted ordinal loss from \cite{yunpengxu2023ordinal} when constructing prediction sets.

Following the same protocol as in the CIFAR-10 experiments, we randomly split the images into training/validation, calibration, and test sets in a 4:1:1 ratio and train a ResNet-34 \cite{He_2016_CVPR} to obtain logits for the five severity levels.

\paragraph{Coverage and risk control.}
Figures~\ref{fig:drd_coverage_control_for_different_xi} and \ref{fig:drd_risk_control_for_different_alpha} show the selective coverage and risk as functions of $\xi$ and $\alpha$.  
The overall behavior closely mirrors that of CIFAR-10: both SCRC-T and SCRC-I achieve the desired selective coverage and risk across all settings, with their selected samples exhibiting substantially smaller prediction sets than rejected samples.  
CRC\_ALL again maintains a fixed coverage of $1.0$, and RAND remains insensitive to $\xi$ due to its uninformed selection step.

\begin{figure}[ht]
\centering
\begin{subfigure}{1.05\linewidth}
  \centering
  \includegraphics[width=\linewidth]{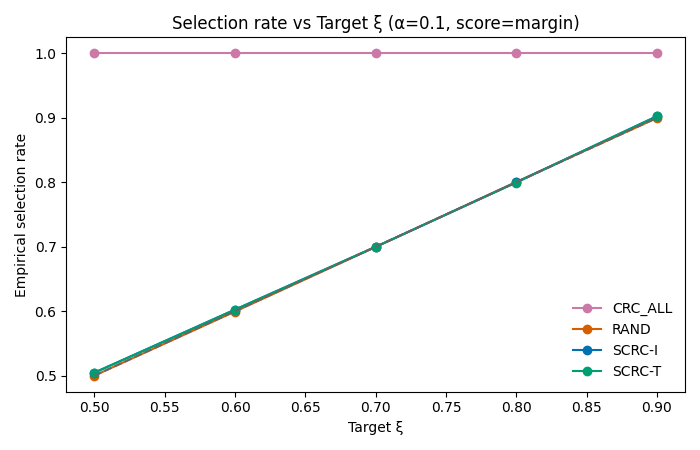}
\end{subfigure}
\begin{subfigure}{1.05\linewidth}
  \centering
  \includegraphics[width=\linewidth]{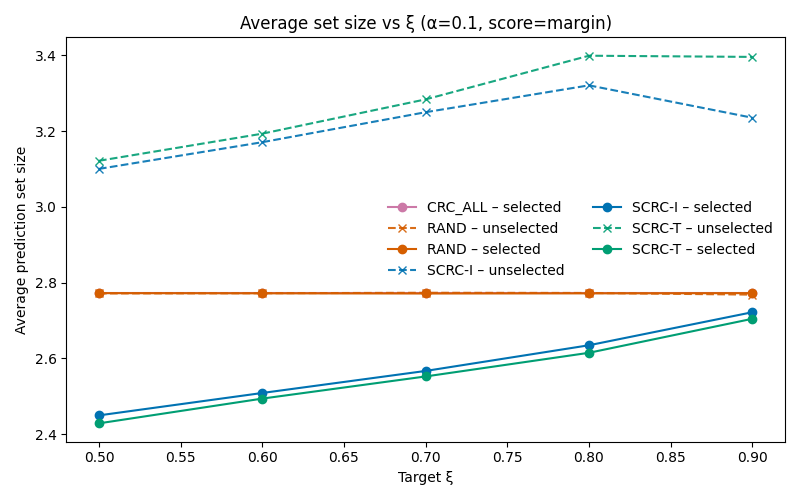}
\end{subfigure}
\caption{DR Detection: Coverage control at different values of $\xi$ with $\alpha=0.1$ (margin score).}
\label{fig:drd_coverage_control_for_different_xi}
\end{figure}

\begin{figure}[ht]
\centering
\begin{subfigure}{1.05\linewidth}
  \centering
  \includegraphics[width=\linewidth]{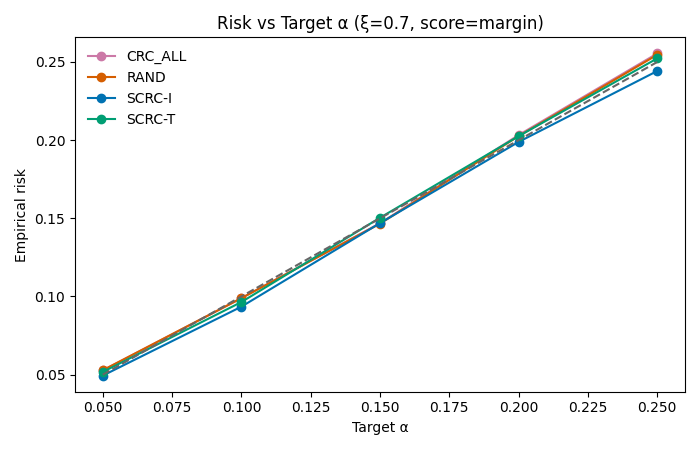}
\end{subfigure}
\begin{subfigure}{1.05\linewidth}
  \centering
  \includegraphics[width=\linewidth]{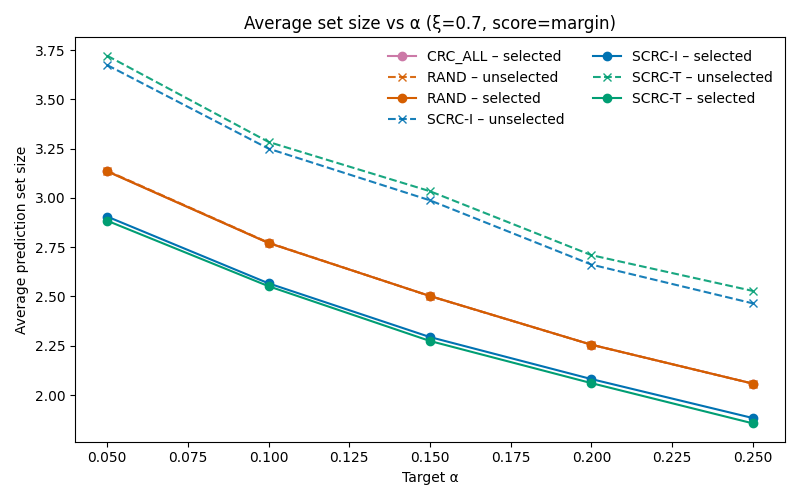}
\end{subfigure}
\caption{DR Detection: Risk control at different values of $\alpha$ with $\xi=0.7$ (margin score).}
\label{fig:drd_risk_control_for_different_alpha}
\end{figure}

\paragraph{Comparison of methods.}
As in the CIFAR-10 results, SCRC-T and SCRC-I perform almost identically, with SCRC-I being slightly more conservative due to its PAC-style correction.  
This conservativeness is more noticeable on smaller calibration sets, which are common in medical datasets, but the effect remains small in practice.  
Importantly, both methods maintain valid uncertainty control despite the ordinal structure of the labels. In this setting, abstention is especially useful because difficult cases can be deferred, while accepted cases still receive informative calibrated set-valued predictions.

\paragraph{Effect of score functions and $\delta$.}
Figure~\ref{fig:drd_score_functions} shows prediction set sizes for various score functions; the relative ranking is consistent with CIFAR-10: entropy and energy yield the smallest sets, and margin is more conservative.  
Table~\ref{tab:cifar10_delta} summarizes the effect of $\delta$ in SCRC-I.  
Smaller $\delta$ values (tighter confidence levels) lead to slightly larger prediction sets and marginal increases in conservativeness, as expected from the tighter DKW constraint.

\begin{figure}[!htb]
  \centering
  \includegraphics[width=1.05\linewidth]{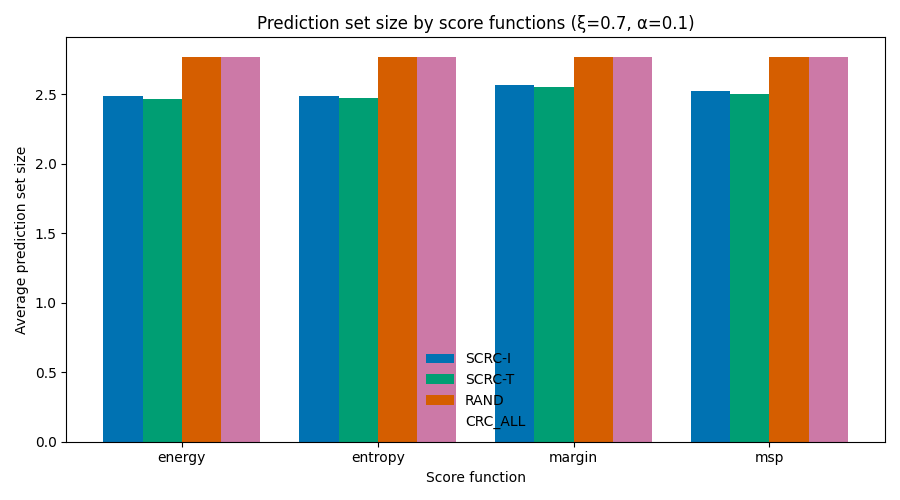}
  \caption{DR Detection: Comparison of different selection score functions.}
  \label{fig:drd_score_functions}
\end{figure}

\begin{table}[ht]
\centering
\caption{DR Detection: Impact of different values of $\delta$ for SCRC-I.}
\label{tab:cifar10_delta}
\setlength{\tabcolsep}{10pt}
\renewcommand{\arraystretch}{1.3}
\begin{tabular}{|l|c|c|c|}
\hline
\textbf{Metric} & \textbf{$\delta=0.01$} & \textbf{$\delta=0.05$} & \textbf{$\delta=0.10$} \\
\hline
Empirical risk        & 0.0931 & 0.0935 & 0.0937 \\
Prediction set size   & 2.569  & 2.567  & 2.566  \\
\hline
\end{tabular}
\end{table}

\section{Conclusions}
In this paper, we introduced \textit{Selective Conformal Risk Control}, a unified framework that integrates conformal prediction with selective classification to provide reliable, distribution-free uncertainty quantification while improving efficiency through selective abstention. The framework targets the selective setting: it produces calibrated prediction sets on accepted inputs and explicitly abstains on rejected ones, rather than attempting to force informative prediction sets on every sample. Our design enables simultaneous guarantees on selective coverage and conditional risk, and addresses the practical limitation of standard conformal methods that often yield excessively large prediction sets. 

We developed two algorithmic variants within this framework: SCRC-T and SCRC-I. Empirical results demonstrate that both methods successfully achieve the target coverage and risk levels.  
Their performance is nearly identical across all tested configurations, with the SCRC-I variant exhibiting slightly more conservative risk control, an expected consequence of its PAC correction. Nevertheless, SCRC-I is more practical for real world deployment, as it avoids parameter recomputation while maintaining reliable uncertainty calibration. In contrast, SCRC-T serves as the exact exchangeability-preserving construction and provides a useful benchmark when per-test recomputation is acceptable.

A limitation of the current framework is that it does not provide a second-stage prediction-set guarantee for rejected samples; these cases are intentionally deferred and should be handled by a downstream fallback mechanism such as human review or a more specialized model. Future work may explore integrating such fallback mechanisms into a unified framework, as well as developing adaptive or learned selection functions. Additional directions include extending the approach to regression and ranking settings, and applying it to large-scale, high-stakes domains where both reliable uncertainty control and computational efficiency are critical.

\printbibliography

\end{document}